\documentclass[conference]{IEEEtran}

\usepackage{cite}
\usepackage{comment}
\usepackage{amsmath,amssymb,amsfonts}
\usepackage{graphicx}
\usepackage{textcomp}
\usepackage{xcolor}
\def\BibTeX{{\rm B\kern-.05em{\sc i\kern-.025em b}\kern-.08em
    T\kern-.1667em\lower.7ex\hbox{E}\kern-.125emX}}

\usepackage[utf8]{inputenc} 
\usepackage[T1]{fontenc}    
\usepackage{hyperref}       
\usepackage{url}            
\usepackage{booktabs}       
\usepackage{amsfonts}       
\usepackage{nicefrac}       
\usepackage{microtype}      
\usepackage{multirow}
\usepackage{subfigure}
\usepackage{amssymb}
\usepackage{amsmath}
\usepackage{mystyle}
\usepackage{amsthm}
\usepackage{algorithm} 
\usepackage{algorithmic}  
\usepackage[algo2e, ruled]{algorithm2e}
\usepackage{float}
\usepackage{xcolor}

\newtheorem{lemma}{Lemma}

\newcommand{\GloVe}{\s{GloVe}\xspace}
\newcommand{\WordToVec}{\s{word2vec}\xspace}
\newcommand{\RAW}{\s{RAW}\xspace}

\newcommand{\AO}{\textsc{AbsoluteOrientation}}
\newcommand{\AOc}{\textsc{AO-Centered}}
\newcommand{\AOR}{\textsc{AO-Rotation}}
\newcommand{\AOs}{\textsc{AO+Scaling}}
\newcommand{\AOsc}{\textsc{AO-Centered+Scaling}}
\newcommand{\AOW}{\textsc{AO+Weighted}}
\newcommand{\AON}{\textsc{AO-Normalized}}
\newcommand{ \AOWts}{\textsc{AO+Centered+Weighted}}

\newcommand{\R}{\ensuremath{\mathbb{R}}}
\newcommand{\SO}{\ensuremath{\mathbb{SO}}}
\newcommand{\OG}{\ensuremath{\mathbb{O}}}

\newcommand{\RG}{\textsc{RG}\xspace}
\newcommand{\WSim}{\textsc{WSim}\xspace}
\newcommand{\MC}{\textsc{MC}\xspace}
\newcommand{\Simlex}{\textsc{Simlex}\xspace}
\newcommand{\SEM}{\textsc{SEM}\xspace}
\newcommand{\SYN}{\textsc{SYN}\xspace}
\newcommand{\update}[1]{#1}

\newcommand{\myParagraph}[1]{\vspace{1mm}\noindent {\textbf{#1}.}}

\newcommand{\ve}[1]{\ensuremath{v_{\textsf{#1}}}}

\begin{document}

\title{Closed Form Word Embedding Alignment\\
%
}

\author{\IEEEauthorblockN{Sunipa Dev}
\IEEEauthorblockA{\textit{School of Computing} \\
Salt Lake City, USA\\
sunipad@cs.utah.edu}
\and
\IEEEauthorblockN{Safia Hassan}
\IEEEauthorblockA{\textit{School of Computing} \\
Salt Lake City, USA\\
safiahassan609@gmail.com}
\and
\IEEEauthorblockN{Jeff M Phillips}
\IEEEauthorblockA{\textit{School of Computing} \\
Salt Lake City, USA\\
jeffp@cs.utah.edu}
}
\maketitle

\begin{abstract}
We develop a family of techniques to align word embeddings which are derived from different source datasets or created using different mechanisms (e.g., GloVe or word2vec).  
Our methods are simple and have a closed form to optimally rotate, translate, and scale to minimize root mean squared errors or maximize the average cosine similarity  between two embeddings of the same vocabulary into the same dimensional space.  
Our methods extend approaches known as Absolute Orientation, which are popular for aligning objects in three-dimensions, and generalize an approach by Smith \etal (ICLR 2017).  
We prove new results for optimal scaling and for maximizing cosine similarity.  		
Then we demonstrate how to evaluate the similarity of embeddings from different sources or mechanisms, and that certain properties like synonyms and analogies are preserved across the embeddings and can be enhanced by simply aligning and averaging ensembles of embeddings.  
\end{abstract}


	\section{INTRODUCTION}

\vspace{-1mm}

Embedding complex data objects into a high-dimensional, but easy to work with, feature space has been a popular paradigm in data mining and machine learning for more than a decade~\cite{RR07,RR08,hashkernels,featurehash}.  
This has been especially prevalent recently as a tool to understand language, with the popularization through \WordToVec~\cite{Mik1,Art5} and \GloVe~\cite{Art3}.  These approaches take as input a large corpus of text, and map each word which appears in the text to a vector representation in a high-dimensional space (typically $d=300$ dimensions).  

These word vector representations began as attempts to estimate similarity between words based on the context of their nearby text, or to predict the likelihood of seeing words in the context of another.  Other more powerful properties were discovered.  Consider each word gets mapped to a vector $\ve{word} \in \R^d$.  

\begin{itemize}
\item \emph{Synonym similarity:}  
Two synonyms (e.g., $\ve{car}$ and $\ve{automobile}$) tend to have small Euclidean distances and large inner products, and are often nearest neighbors.  

\item \emph{Linear relationships:}  
For instance, the vector subtraction between countries and capitals (e.g., $\ve{Spain}-\ve{Madrid}$, $\ve{France}-\ve{Paris}$, $\ve{Germany}-\ve{Berlin}$) are similar.  Similar vectors encode gender (e.g., $\ve{man}-\ve{woman}$), tense ($\ve{eat}-\ve{ate}$), and degree ($\ve{big}-\ve{bigger}$).  

\item \emph{Analogies:} 
The above linear relationships could be transferred from one setting to another.  For instance the gender vector $\ve{man}-\ve{woman}$ (going from a female object to a male object) can be transferred to another more specific female object, say $\ve{queen}$.  Then the result of this vector operation is $\ve{queen} + (\ve{man} - \ve{woman})$ is close to the vector $\ve{king}$ for the word ``\s{king}.''   This provides a mechanism to answer analogy questions such as ``\s{woman:man::queen:}?'' 

\item \emph{Classification:}  
More classically~\cite{RR07,RR08,hashkernels,featurehash}, one can build linear classifiers or regressors to quantify or identify properties like sentiment.  
\end{itemize}

At least in the case of \GloVe, these linear substructures are not accidental; the embedding aims to preserve inner product relationships.  Moreover, these properties all enforce the idea that these embeddings are useful to think of inheriting a Euclidean structure, i.e., it is safe to represent them in $\R^d$ and use Euclidean distance.  

However, there is nothing extrinsic about any of these properties.  A rotation or scaling of the entire dataset will not affect synonyms (nearest neighbors), linear substructures (dot products), analogies, or linear classifiers.  A translation will not affect distance, analogies, or classifiers, but will affect inner products since it effectively changes the origin.  These substructures (i.e., metric balls, vectors, halfspaces) can be transformed in unison with the embedded data.  
Indeed Euclidean distance is the only metric on $d$-dimensional vectors that is rotation invariant.

The intrinsic nature of these embeddings and their properties adds flexibility that can also be a hinderance.  In particular, we can embed the same dataset into $\R^d$ using two approaches, and these structures cannot be used across datasets.  Or two data sets can both be embedded into $\R^d$ by the same embedding mechanism, but again the substructures do not transfer over.  That is, the same notions of similarity or linear substructures may live in both embeddings, but have different meaning with respect to the coordinates and geometry.  
This makes it difficult to compare approaches; the typical way is to just measure a series of accuracy scores, for instance in recovering synonyms~\cite{Lev2,Mik1}.  However, these single performance scores do not allow deeper structural comparisons.  

%

Another issue is that it becomes challenging (or at least messier) to build ensemble structures for embeddings.  For instance, some groups have built word vector embeddings for enormous datasets (e.g., \GloVe embedding using 840 billion tokens from Common Crawl, or the \WordToVec embedding using 100 billion tokens of Google News), which costs at least tens of thousands of dollars in cloud processing time.  Given several such embeddings, how can these be combined to build a new single better embedding without revisiting that processing expense?  How can a new (say specialized) data set from a different domain use a larger high-accuracy embedding?

\myParagraph{Our approach and results}
In this paper we provide a simple closed form method to optimally align two embeddings.  These methods find optimal rotation (technically an orthogonal transformation) of one dataset onto another, and can also solve for the optimal scaling and translation.  
They are optimal in the sense that they minimize the sum of squared errors under the natural Euclidean distance between all pairs of common data points, or they can maximize the average cosine similarity.  

The methods we consider are easy to implement, and are based on $3$-dimensional shape alignment techniques common in robotics and computer vision called ``absolute orientation.''
We observe that these approaches extend to arbitrary dimensions $d$; the same solution for the optimal orthogonal transformation was also recently re-derived by Smith \etal (2017)~\cite{SmithTHH17}.  

In this paper, we also show that an approach to choose the optimal scaling of one dataset onto another~\cite{Hor87} does not affect the optimal choice of rotation.  Hence, the choice of translation, rotation, and scaling can all be derived with simple closed form operations.  

We then apply these methods to align various types of word embeddings, providing new ways to compare, translate, and build ensembles of them.  
We start by aligning data sets to themselves with various types of understandable noise; this provides a method to calibrate the error scores reported in other settings.  
We also demonstrate how these aligned embeddings perform on various synonym and analogy tests, whereas without alignment the performance is very poor.  The results with scaling, translation, and weighting all consistently improve upon the results for only rotation as advocated by Smith \etal (2017)~\cite{SmithTHH17}.  

Moreover, we show that we can boost embeddings, showing improved results when aligning various embeddings, and taking simple averages of the embedded words from different data sets.  The results from these boosted embeddings provide the best known results for various analogy and synonym tests.  
More extensive use of ensembles should be possible, and it could be applied to a wider variety of data types where Euclidean feature embeddings are known, such as for graphs~\cite{DeepWalk,node2vec,CZC17,GF17,DCS17}, images~\cite{SIFT,SURF}, and for kernel methods~\cite{RR07,RR08}. 

This alignment can also aid translation, wherein an alignment learned from a small set of words whose translation is known, we can obtain an alignment of a much larger set of words.  We also show how aligning two low-resource languages independently to a well-documented and accurate intermediate language can aid in translation between the first two languages. 

Finally, in the last few years, contextualized embeddings, such as BERT \cite{bert} and ELMo \cite{Peters:2018}, which embed a word differently each time, based on the context it appears in, have become increasingly pervasively used in language processing tasks such as textual entailment and co-reference resolution.  We show that a simple average of the contexts allows our techniques to efficiently extend to modeling a more complex multi-way alignment among word representations.

\update{
\subsection{Word Embedding Mechanisms}
There are several different mechanisms today to embed words into a high dimensional vector space. They can primarily be divided into mechanisms: first, those that produce a single vector for each word (e.g., \GloVe~\cite{Art3}), thus leading to a dictionary like structure, and the second (ELMO~\cite{Peters:2018}, BERT~\cite{bert}) producing a function instead of a vector for a word such that given a context, the vector for the word is generated. This implies that, different word senses (river $bank$ versus financial institution $bank$) lead to differently embedded representations of the same word. 
Our experiments primarily examine the first kind as vector distances are interpretable there.}

\update{
The word embedding mechanisms we use here are:} 

\update{\textbf{\RAW:} has as many dimensions as there are words, each dimension corresponds with the rate of co-occurrences with a particualr word; it is in some sense what other sophisticated models such as \GloVe are trying to understand and approximate at much lower dimensions. }

\update{\textbf{\GloVe:} uses an unsupervised learning algorithm~\cite{Art3} for obtaining vector representations for words based on their aggregated global word-word co-occurrence statistics from a corpus.}

\update{\textbf{\WordToVec:} builds representations of words so the cosine similarity of their embeddings can be used to predict a word that would fit in a given context and can be used to predict the context that would be an appropriate fit for a given word.  }

\update{\textbf{FastText:} scales~\cite{arm2016bag} these methods of deriving word representations to be more usable with larger data with less time cost. The also include sub-word information~\cite{bojanowski2016enriching} to be able to generate word embeddings for words unseen during training. We use FastText word representations for Spanish and French from the library provided (\url{https://fasttext.cc/docs/en/crawl-vectors.html}) for our translation experiments in Section \ref{app : langauges}.}

\update{More recently, contextual embeddings which produce vectors for word in every distinct context have become popular, such as ELMO~\cite{Peters:2018} and BERT~\cite{bert}. These embeddings differ from the other embeddings mentioned above in that they do not give a look-up table/dictionary in the end wherein each word has one exact corresponding vector representation. We describe an extension of our method of Absolute Orientation for alignment for these embeddings in Section \ref{sec : contextual}.}


\section{CLOSED FORM POINT SET ALIGNMENT: CLASSIC AND NEW RESULTS}

\label{sec:AO}

In many classic computer vision and shape analysis problems, a common problem is the alignment of two (often $3$-dimensional) shapes.  The most clean form of this problem starts with two points sets $A = \{a_1, a_2, \ldots, a_n\}$ and $B = \{b_1, b_2, \ldots, b_n\}$, each of size $n$, where each $a_i$ corresponds with $b_i$ (for all $i \in 1,2,\ldots,n$).  Generically we can say each $a_i, b_i \in \R^d$ (without restricting $d$), but as mentioned the focus of this work was typically restricted to $d=3$ or $d=2$.  Then the standard goal was to find a rigid transformation -- a translation $t \in \R^d$ and rotation $R \in \SO(d)$ -- to minimize the root mean squared error (RMSE).  
An equivalent formulation is to solve for the sum of squared errors as
\begin{equation}\label{eq:opt-R+t}
(R^*, t^*)  = \argmin_{t \in \R^d, R \in \SO(d)} \sum_{i=1}^n \|a_i - (b_i R + t)\|^2.
\end{equation}
For instance, this is one of the two critical steps in the well-known iterative closest point (ICP) algorithm~\cite{BM92,CM92}.  

In the 1980s, several \emph{closed form} solutions to this problem were discovered; their solutions were referred to as solving \emph{absolute orientation}.  
The most famous paper by Horn~\cite{Hor87} uses unit quaternions.  However, this approach seems to have been known earlier~\cite{FH83}, and other techniques using rotation matrices and the SVD~\cite{HN81,AHB87}, rotation matrices and an eigen-decomposition~\cite{SS87,Sch66},  and dual number quaternions~\cite{WSV91}, have also been discovered.  In $2$ or $3$ dimensions, all of these approaches take linear (i.e., $O(n)$) time, and in practice have roughly the same run time~\cite{ELF97}.

In this document, we focus on the Singular Value Decomposition(SVD)-based approach of Hanson and Norris~\cite{HN81}, since it is clear, has an easy analysis, and unlike the quaternion-based approaches which only work for $d=3$, generalizes to any dimension $d$. \update{A singular value decomposition(SVD) factorizes a matrix of dimensions $m \times n$ to produce two orthonormal matrices ($U$ and $V$) and a diagonal matrix ($S$) to satisfy the linear transformation $x = Ax$. The orthonormal matrices capture the rotation or reflection of the space while the diagonal matrix $S$ captures the singular values which interpret the magnitude of information along each of the respective dimensions.
Hanson and Norris's approach decouple the rotation from the translation and solve for each independently. It further uses the orthonormal matrices produced by SVD to determine the rotation.}  In particular, this approach first finds the means $\bar a = \frac{1}{n} \sum_{i=1}^n a_i$ and $\bar b = \frac{1}{n} \sum_{i=1}^n b_i$ of each data set.  Then it creates centered versions of those data sets $\hat A \leftarrow (A,\bar a)$ and $\hat B \leftarrow (B, \bar b)$.  
%
Next we need to compute the RMSE-minimizing rotation (all rotations are then considered around the origin) on centered data sets $\hat A$ and $\hat B$.  First compute the sum of outer products $H = \sum_{i=1}^n \hat b_i^T \hat a_i$, which is a $d \times d$ matrix.  We emphasize $\hat a_i$ and $\hat b_i$ are row vectors, so this is an outer product, not an inner product.  Next take the singular value decomposition of $H$ so $[U, S, V^T] = \mathsf{svd}(H)$, and the ultimate rotation is $R = U V^T$.  We can create the rotated version of $B$ as $\tilde B = \hat B R$ so we rotate each point as $\tilde b_i = \hat b_i R$.  



Within this paper we will use this approach, as outlined in Algorithm \ref{alg:AO-rotate}, to align several data sets each of which have no explicit intrinsic properties tied to their choice of rotation.  
We in general do not use the translation step for two reasons.  First, this effectively changes the origin and hence the inner products.  Second, we observe the effect of translation is usually small, and typically does not improve performance.  

\begin{algorithm}
	\caption{\label{alg:AO-rotate}$\AOR(A,B)$}
	\begin{algorithmic}
		\STATE Compute the sum of outer products $H = \sum_{i=1}^{n} b_i^T a_i$   
		\STATE Decompose  $[U, S, V^T] = \mathsf{svd}(H)$ 
		\STATE Build rotation $R = U V^T$   
		\STATE \textbf{return} $\tilde B = B R$ so each $\tilde b_i = b_i R$  
	\end{algorithmic}
\end{algorithm}

Technically, this may allow $R$ to include mirror flips, in addition to rotations.  These can be detected (if the last singular value is negative) and factored out by multiplying by a near-identity matrix $R = U I_- V^T$ where $I_-$ is identity, except the last term is changed to $-1$.  We ignore this issue in this paper, and henceforth consider orthogonal matrices $R \in \OG(d)$ (which includes mirror flips) instead of just rotations $R \in \SO(d)$.  For simpler nomenclature, we still refer to $R$ as a ``rotation.''

We discuss here a few other variants of this algorithm which take into account translation and scaling between $A$ and $B$.
 
\begin{algorithm}
	\caption{\label{alg:AO}\AO$(A,B)$ \cite{HN81}}
	\begin{algorithmic}
		\STATE Compute $\bar a = \frac{1}{n} \sum_{i=1}^n a_i$ and $\bar b = \frac{1}{n} \sum_{i=1}^n b_i$  
		\STATE \underline{Center} $\hat A \leftarrow (A,\bar a)$ so each $\hat a_i = a_i - \bar a$, and similarly $\hat B \leftarrow (B, \bar b)$   
		\STATE Compute the sum of outer products $H = \sum_{i=1}^{n} \hat b_i^T \hat a_i$   
		\STATE Decompose  $[U, S, V^T] = \mathsf{svd}(H)$ 
		\STATE Build rotation $R = U V^T$   
		\STATE \underline{Rotate} $\tilde B = \hat B R$ so each $\tilde b_i = \hat b_i R$  
		\STATE Translate $B^* \leftarrow (\tilde B, - \bar a)$ so each $b^*_i = \tilde b_i + \bar a$  
		\STATE \textbf{return} $B^*$
	\end{algorithmic}
\end{algorithm}

Note that the rotation $R$ and translation $t = - \bar b + \bar a$ derived within this Algorithm \ref{alg:AO} are not exactly the optimal $(R^*, t^*)$ desired in formulation (\ref{eq:opt-R+t}).  This is because the order these are applied, and the point that the data set is rotated around is different.  In formulation (\ref{eq:opt-R+t}) the rotation is about the origin, but the dataset is not centered there, as it is in Algorithm \ref{alg:AO}.

\myParagraph{Translations}
To compare with the use of also optimizing for the choice of translations in the transformation, we formally describe this procedure here.  In particular, we can decouple rotations and translations, so to clarify the discrepancy between Algorithm \ref{alg:AO} and equation (\ref{eq:opt-R+t}), we use a modified version of the above procedure.  In particular, we first center all data sets, $\hat A \leftarrow A$ and $\hat B \leftarrow B$, and henceforth can know that they are already aligned by the optimal translation.  Then, once they are both centered, we can then call $\AOR(\hat A, \hat B)$.  
This is written explicitly and self-contained in Algorithm \ref{alg:AO-cent}.

\begin{algorithm}
	\caption{\label{alg:AO-cent}$\AOc(A,B)$}
	\begin{algorithmic}
		\STATE Compute $\bar a = \frac{1}{n} \sum_{i=1}^n a_i$ and $\bar b = \frac{1}{n} \sum_{i=1}^n b_i$  
		\STATE \underline{Center} $\hat A \leftarrow (A,\bar a)$ so each $\hat a_i = a_i - \bar a$, and similarly $\hat B \leftarrow (B, \bar b)$   
		\STATE Compute the sum of outer products $H = \sum_{i=1}^{n} \hat b_i^T \hat a_i$   
		\STATE Decompose  $[U, S, V^T] = \mathsf{svd}(H)$ 
		\STATE Build rotation $R = U V^T$   
		\STATE \underline{Rotate} $\tilde B = \hat B R$ so each $\tilde b_i = \hat b_i R$  
		\STATE \textbf{return} $\hat A$, $\tilde B$
	\end{algorithmic}
\end{algorithm}

%
%
%

\myParagraph{Scaling}
In some settings, it makes sense to align data sets by scaling one of them to fit better with the other, formulated as  
$
(R^*, t^*, s^*)  = \argmin_{s \in \R, R \in \SO(d)} \sum_{i=1}^n \|a_i - s (b_i-t) R \|^2.
$
In addition to the choices of translation and rotation, the optimal choice of scaling can also be decoupled.  

Horn \etal~\cite{Hor87} introduced two mechanisms for solving for a scaling that minimizes RMSE.  Assuming the optimal rotation $R^*$ has already been applied to obtain $\hat B$, then a closed form solution for scaling is 
$
s^* 
= \sum_{i=1}^n \langle \hat a_i, \hat b_i \rangle / \|\hat B\|_F^2.  
$
The sketch for Absolute Orientation with scaling, is in Algorithm \ref{alg:AO+scale}.

\begin{algorithm}
	\caption{\label{alg:AO+scale}$\AOs(A,B)$}
	\begin{algorithmic}
		\STATE $\tilde B \leftarrow \AOR(A,B)$
		\STATE Compute scaling $s = \sum_{i=1}^n \langle a_i, \tilde b_i \rangle / \|\tilde B\|_F^2$
		\STATE \textbf{return} $\breve B$ as $\breve B \leftarrow s \tilde B$ so for each $\breve b_i = s \tilde b_i$.  
	\end{algorithmic}
\end{algorithm}

The steps of rotation, scaling and translation fit together to give us algorithm \ref{alg:AO-scale}. 

\begin{algorithm}
	\caption{\label{alg:AO-scale}$\AOsc(A,B)$}
	\begin{algorithmic}
		\STATE Compute $\bar a = \frac{1}{n} \sum_{i=1}^n a_i$ and $\bar b = \frac{1}{n} \sum_{i=1}^n b_i$  
		\STATE \underline{Center} $\hat A \leftarrow (A,\bar a)$ so each $\hat a_i = a_i - \bar a$, and similarly $\hat B \leftarrow (B, \bar b)$   
		\STATE Compute the sum of outer products $H = \sum_{i=1}^{n} \hat b_i^T \hat a_i$   
		\STATE Decompose  $[U, S, V^T] = \mathsf{svd}(H)$ 
		\STATE Build rotation $R = U V^T$   
		\STATE \underline{Rotate} $\tilde B = \hat B R$ so each $\tilde b_i = \hat b_i R$  
		\STATE Compute scaling $s = \sum_{i=1}^n \langle a_i, b_i \rangle / \|B\|_F^2$
		\STATE \underline{Scale} $\breve B$ as $\breve B \leftarrow s \tilde B$ so for each $\breve b_i = s \tilde b_i$.  
		\STATE \textbf{return} $\tilde A, \breve B$
	\end{algorithmic}
\end{algorithm}


Horn \etal~\cite{Hor87} presented an alternative closed form choice of scaling $s$ which minimizes RMSE, but under a slightly different situation.  In this alternate formulation, $A$ must be scaled by $1/s$ and $B$ by $s$, so the new scaling is somewhere in the (geometric) middle of that for $A$ and $B$.  We found this formulation less intuitive, since the RMSE is dependent on the scale of the data, and in this setting the new scale is aligned with neither of the data sets.  
%
However, Horn \etal~\cite{Hor87} only showed that the choice of optimal scaling is invariant from the rotation in the second (less intuitive) variant.  We present a proof that this rotation invariance also holds for the first variant.  The proof uses the structure of the SVD-based solution for optimal rotation, with which Horn \etal may not have been familiar.

\begin{lemma}\label{lem:scaling}
Consider two points sets $A$ and $B$ in $\R^d$.  
After the rotation and scaling in Algorithm \ref{alg:AO+scale}, no further rotation about the origin of $\breve B$ can reduce the RMSE.  
\end{lemma}
\vspace{-4mm}
\begin{proof}
	We analyze the SVD-based approach we use to solve for the new optimal rotation.  Since we can change the order of multiplication operations of $s b_i R$, i.e. scale then rotate, we can consider first applying $s^*$ to $B$, and then re-solving for the optimal rotation.  Define $\check B = s^* B$, so each $\check b_i = s^* b_i$.  
Now to complete the proof, we show that the optimal rotation $\check R$ derived from $A$ and $\check B$ is the same as was derived from $A$ and $B$.  
	
Computing the outer product sum
$
	\check H = \sum_{i=1}^n \check b_i^T a_i = \sum_{i=1}^n (s^* b_i)^T a_i = s^* \sum_{i=1}^n b_i^T a_i = s^* H,
$
is just the old outer product sum $H$ scaled by $s^*$.  Then its SVD is 
$
	\mathsf{svd}(\check H) \rightarrow [\check U, \check S, \check V^T] = [U, s^* S, V^T],
$
since all of the scaling is factored into the $S$ matrix.  Then since the two orthogonal matrices $\check U = U$ and $\check V = V$ are unchanged, we have that the resulting rotation 
$
	\check R = \check U \check V^T = U V^T = R
$
is also unchanged.  
\end{proof}

\myParagraph{Preserving Inner Products}
%
While Euclidean distance is a natural measure to preserve under a set of transformations, many word vector embeddings are evaluated or accessed by Euclidean inner product operations.  It is natural to ask if our transformations also maximize the sum of inner products of the aligned vectors.  Or does it maximize the sum of cosine similarity: the sum of inner products of \emph{normalized} vectors.  Indeed we observe that $\AOR(A,B)$ results in a rotation 
$
	\tilde R  = \argmax_{R \in \SO(d)} \sum_{i=1}^n \langle a_i, b_i R\rangle.   
$

\begin{lemma}\label{lem:inner}
	$\AOR(A,B)$ rotates $B$ to $\tilde B$ to maximize $\sum_{i=1}^n \langle a_i, \tilde b_i\rangle$. If $a_i \in A$ and $b_i \in B$ are normalized $\|a_i\| = \|b_i\|=1$, then the rotation maximizes the sum of cosine similarities $\sum_{i=1}^n \left\langle \frac{a_i}{\|a_i\|}, \frac{\tilde b_i}{\|\tilde b_i\|}\right \rangle$.  
\end{lemma}
\begin{proof}
	From Hanson and Norris~\cite{HN81} we know $\AOR(B)$ finds a rotation $R^*$ so
	\[ 
	R^*  = \argmin_{ R \in \SO(d)} \sum_{i=1}^n \|a_i - (b_i R )\|^2.
	\]
	Expanding this equation we find 
	\[
	R^*  = \argmin_{ R \in \SO(d)} \left( \sum_{i=1}^n \|a_i\|^2 - \sum_{i=1}^n 2\langle a_i,b_iR\rangle + \sum_{i=1}^n \|b_i R\|^2 \right).  
	\]
	\update{Now, the length of a vector does not change upon rotation(R), thus, $\|b_iR\|^2 = \|b_i\|^2$. So, since $\|a_i\|^2$ and  $\|b_i\|^2$ are both lengths of vectors and thus, properties of the dataset, they do not depend on the choice of $R$} and as desired 
	\[
	R^* = \argmax_{ R \in \SO(d)}  \sum_{i=1}^n \langle a_i, b_iR\rangle.
	\]
	
	If all $a_i, b_i$ are normalized, then $R$ does not change the norm $\|\tilde b_i\| = \|b_i R\| = \|b_i\| =1$.  So for $\tilde b_i = b_i R$, each 
	$\langle a_i, \tilde b_i\rangle = \langle \frac{a_i}{\|a_i\|}, \frac{\tilde b_i}{\|\tilde b_i\|} \rangle$ 
	and hence, as desired, 
	\[
	R^* = \argmax_{R \in \SO(d)} \sum_{i=1}^n \left\langle \frac{a_i}{\|a_i\|}, \frac{b_i R}{\|b_i R\|}\right \rangle. \qedhere
	\]
\end{proof}

Several evaluations of word vector embeddings use cosine similarity, so it suggests first normalizing all vectors $a_i \in A$ and $b_i \in B$ before performing $\AOR(A,B)$.  However, we found this does not empirically work as well.  The rational is that vectors with larger norm tend to have less noise and are supported by more data.  So the unnormalized alignment effectively weights the importance of aligning the inner products of these vectors more in the sum, and this leads to a more stable method.  Hence, in general, we do not recommend this normalization preprocessing.

\subsection{Extension to Contextualized Embeddings}
\label{sec : contextual}
In recent years, contextualized embeddings such as ELMo \cite{Peters:2018} and BERT \cite{bert} have become increasingly popular, because of their ability to express the polysemity of words. A word in these frameworks is not expressed as a single vector, but rather, based on its different meanings or different contexts it has been used in.  That is, each instance of a word is represented by a different vector in the embedding space.  Our method to align individual vectors does not directly apply in this scenario.  

We propose a simple extension to handle this scenario.  Given a word $w_i$ with two separate contextual embeddings, let these embedding vectors be two sets $A_i = \{a_{i,1}, a_{i,2}, \ldots a_{m_{A,i}}\}$ and $B_i = \{b_{i,1}, b_{i,2},  \ldots, b_{m_{B,i}}\}$ of sizes $m_{A,i}$ and $m_{B,i}$, respectively.  Then our method, instead of aligning a single pair of vectors for each word, it aligns \emph{all} vector pairs for each word.  For instance, for finding the optimal rotation, this involves an alignment for $n$ words, each $i$th word $w_i$, then the outer product matrix is defined
\[
H = \sum_{i=1}^n \frac{1}{m_{A,i} m_{B,i}} \sum_{j =1}^{m_{A,i}} \sum_{j'=1}^{m_{B,i}} a_{i,j}^T b_{i,j'},
\]
where each set of all-pairs is weighted equally for each $i$ (this is accomplished by dividing by the number of such pairs $m_{A,i} m_{B,i}$.)

This all-pairs alignment can be computationally expensive as the number of instances of each word $m_{A,i}$ and $m_{B,i}$ increase; even if we only use $10$ instances of each word, in each embedding, this increases the number of alignments by a factor $100$.  
However, we observe, in each step the set $A_i$ and $B_i$ can be replaced by their averages
\[
 \bar a_i = \frac{1}{m_{A_i}} \sum_{j=1}^{m_{A,i}} a_{i,j}
 \text{ and }
 \bar b_i = \frac{1}{m_{B_i}} \sum_{j=1}^{m_{B,i}} b_{i,j}.
\]
Then the overall means $\bar b$, $\bar a$, outer product $H$, and scaling $s$ are the same using all instances or the mean instance.  

\begin{lemma}\label{lem:allpair}
The alignments found using all-pair alignment when each word has multiple instances in each embedding is equivalent to that computed by aligning the averages of each set of instances.  
\end{lemma}
\begin{proof}
We need to analyze the $4$ quantities computed as part of any transformation: the two averages, the outer product, and the scale.  
In short, these are all linear vector operations (sum, outer product, inner product), so a vector average can be factored out.  

For each average 
\[
\bar a = \frac{1}{n} \sum_{i=1}^n \frac{1}{m_{A,i}} \sum_{j=1}^{m_{A,i}} a_{i,j} = \frac{1}{n} \sum_{i=1}^n \bar a_i,
\]
and similarly for $\bar b$, the calculations are equivalent.  

For the outer product 
\begin{align*}
H &= \sum_{i=1}^n \frac{1}{m_{A,i} m_{B,i}} \sum_{j =1}^{m_{A,i}} \sum_{j'=1}^{m_{B,i}} a_{i,j}^T b_{i,j'}
\\ &= 
\sum_{i=1}^n  \left( \frac{1}{m_{A,i}} \sum_{j =1}^{m_{A,i}} a_{i,j} \right)^T \left( \frac{1}{ m_{B,i}} \sum_{j'=1}^{m_{B,i}} b_{i,j'} \right)
\\ &= 
\sum_{i=1}^n  \bar a_i^T \bar b_i.
\end{align*}

And finally for the scale
\begin{align*}
s 
&= 
\sum_{i=1}^n \frac{1}{m_{A,i} m_{B,i}} \sum_{j =1}^{m_{A,i}} \sum_{j'=1}^{m_{B,i}} \langle a_{i,j}, b_{i,j'} \rangle / \|B\|_F^2
\\ &=
\sum_{i=1}^n \left \langle \frac{1}{m_{A,i}} \sum_{j =1}^{m_{A,i}} a_{i,j},   \frac{1}{m_{B,i}} \sum_{j'=1}^{m_{B,i}}  b_{i,j'} \right \rangle / \|B\|_F^2
\\ &= 
\sum_{i=1}^n \langle \bar a_i, \bar b_i \rangle / \|B\|_F^2,
\end{align*}
where 
\[
\|B\|_F^2 
= 
\sum_{i=1}^n \left \| \frac{1}{m_{B,i}} \sum_{j=1}^{m_{B,i}} b_{i,j} \right\|^2
=
\sum_{i=1}^n \|\bar b_i\|^2.
\]
Note that the normalization term $\|B\|_F^2$ is defined on the average sum of instances for the all-pairs version, since this is a quadratic operation, and otherwise does not factor out.  
\end{proof}

\subsection{Related Approaches}

\label{sec:related}
As mentioned, Smith \etal (2017)~\cite{SmithTHH17} use Algorithm \ref{alg:AO-rotate} to align \WordToVec word embeddings on English and Italian corpuses, and show that this simple approach is effective in translation. Our work can be seen as building on this, in that we show how to interpret the intrinsic accuracy of such an alignment, how to align word vector corpuses created by different mechanisms, and when to use which variant of the closed form solutions.  Additionally, we confirm some of their language translation results and show that it extends to when the embedding mechanisms for the different language corpuses are not the same (e.g., one by \WordToVec and one by \GloVe), as demonstrated in Section \ref{app : langauges}.  

There are several other methods in the literature which attempt to jointly compute embeddings of datasets so that they are aligned, for instance in jointly embedding corpuses in multiple languages~\cite{Herm,Mik3}.  The goal of the approaches we study is to circumvent these more complex joint alignments. A couple of very recent papers propose methods to align embeddings after their construction, but focus on \emph{affine transformations}, as opposed to the more restrictive but distance preserving rotations of our method.  
Bollegala \etal (2017)~\cite{Art1} uses gradient descent, for parameter $\gamma$, to directly optimize
\[
\argmin_{M \in \R^{d \times d}} \sum_{i=1}^n \|a_i - b_i M \|^2 + \gamma \|M\|^2_F.  
\]
%


Another approach, by Sahin \etal (2017)~\cite{Sahin} uses Low Rank Alignment (LRA), an extension of aligning manifolds from LLE~\cite{Boucher}.  This approach has a 2-step but closed form solution to find an affine transformation applied to both embeddings simultaneously. 
Neither approach directly optimizes for the optimal transformation, and requires regularization parameters; this implies if embeddings start far apart, they remain further apart than if they start closer.  
Both find affine transformations $M$ over $\R^{d \times d}$, not a rotation over the non-convex $\OG(d)$ as does our approach.  This changes the Euclidean distance found in the original embedding to a Mahalanobis distance that will change the order of nearest neighbors under Euclidian and cosine distance.  
Finally, the LRA approach, requires an eigendecomposition of an $2n \times 2n$ matrix, where as ours only requires this of a $d \times d$ matrix, so LRA is far less scalable.


\begin{figure*}
	\includegraphics[width =.27\linewidth]{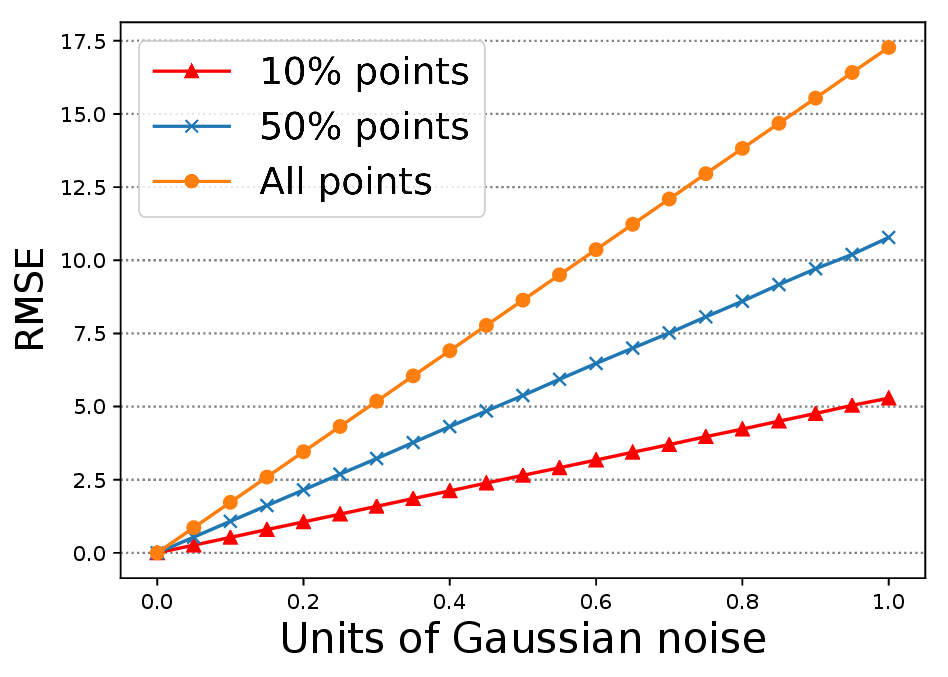} \;\;\;
	\includegraphics[width =.33\linewidth]{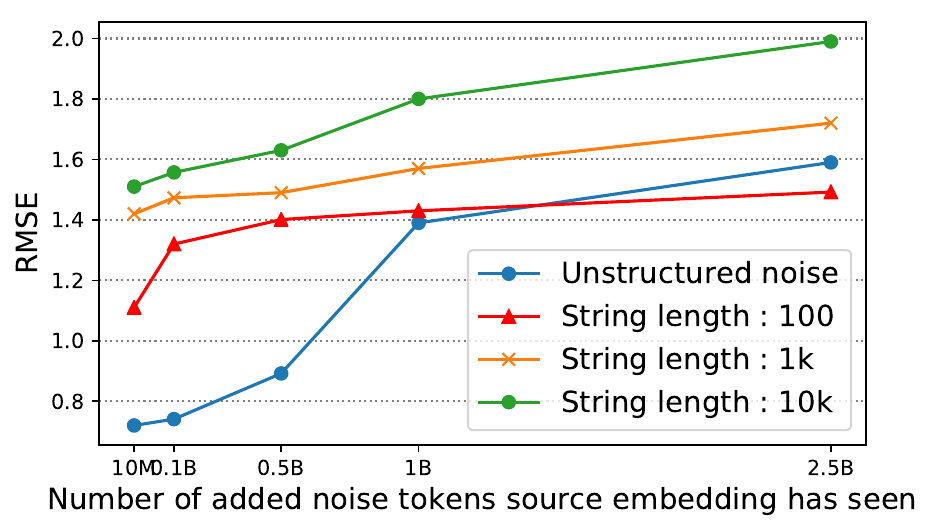} \;\;\;
	\includegraphics[width=.36\linewidth]{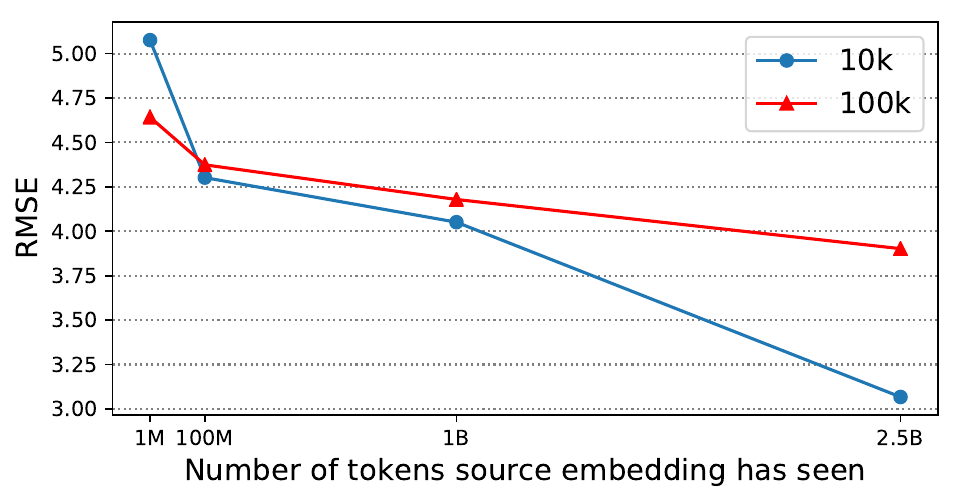} 
	
	\vspace{-3mm}
	\caption{RMSE Error after noise and \AOR{} alignment: 
		Left: adding Gaussian noise to $10\%$, $50\%$ or all points.
		Middle: adding structured and unstructured noise before embedding. 
		Right: incrementally added data.  
		\vspace{-2mm}} 
	\label{fig:RMSE-noise}
\end{figure*}

\section{EVALUATING ACCURACY OF VARIANTS}

\label{sec:exp}
We evaluate the effectiveness of our methods on a variety of scenarios, typically using the Root Mean Square Error:  
$ 
\mathsf{RMSE}(A,B) = \sqrt{\frac{1}{|A|}\sum_{i=1}^{|A|} \| a_{i} - b_{i}\|^2}.
$
We fix the embedding dimension of each $A$ (the target) and $B$ (the source) at $300$, and assume $|A| = |B| = n=100{,}000$ or in some cases $n' = |A| = |B| = 10{,}000$.

We consider embeddings with \GloVe~\cite{Art1} (our default), or \WordToVec~\cite{Mik1,Art5} with Gensim~\cite{rehurek_lrec}, or occasionally \RAW which is just the $L^1$ normalized word count vectors embedded with SVD~\cite{Lev1}. 
We obtain all these three embeddings for our experiments using our default dataset is the \update{$4.57$ billion} token English Wikipedia dump, which we found to be made up of $243K$ vocabulary words (distinct tokens). \update{For \WordToVec, the Gensim~\cite{rehurek_lrec} library provides code for obtaining embeddings of a desired dimensionality, and for \GloVe, the code~\cite{Art1} is provided by the authors themselves. To obtain \RAW embeddings, we run a simple bag of words model which enumerates for each word, how many times it appeared with other words in the vocabulary in a sentence, to give us a vector representation for the word. The \RAW word vectors, thus have the same dimension as that of the vocabulary itself. This when normalized captures the pointwise mutual information and is called the Pointwise Mutual Information (PMI) Matrix.
After embedding using each of these three mechanisms, we select the top 100K most frequent words and their corresponding embeddings for our experiments.}

We compare against existing \GloVe embeddings of
Wikipedia + Gigaword (\s{G(WG)}, \update{6 billion} tokens, 400K vocab), 
Common Crawl (\s{G(CC42)}, \update{42 billion} tokens, 1.9M vocab), and
Common Crawl (\s{G(CC840)}, \update{840 billion} tokens, 2.2M vocab), 
and the existing \WordToVec embedding 
of Google News (\s{W(GN)}, \update{100 billion} tokens, 3 million vocab). \update{All of these embeddings are available online (https://nlp.stanford.edu/projects/glove/, https://code.google.com/archive/p/word2vec/) and were downloaded.} 

When aligning \GloVe embeddings to other \GloVe embeddings we use \AOR.  When aligning embeddings from different sources we use \AOs. 





\myParagraph{Default Data Settings} 
In each embedding, we always consider a consistent vocabulary of $n=100{,}000$ words.  
To decide on this set, we found the $n$ most frequent words used in the default Wikipedia dataset and that were embedded by \GloVe. 
In one case, we compare against smaller datasets and then only work with a small vocabulary of size $n' = 10{,}000$ found the same way.  

For each embedding we represent each word as a vector of dimension $d=300$.  \update{Note that \RAW originally uses an $n$-dimensional vector.  We reduce this to $d$-dimensions by taking its SVD, as advocated by Levy \etal \cite{Lev1}. They demonstrate how \WordToVec implicitly captures the information as the Shifted Pointwise Mutual Information Matrix (SPMI) in low dimensions. They further demonstrate that computing the SVD of the SPMI matrix maintains the structure captured by the full dimensional matrix.}


%
%
%

\subsection{Calibrating RMSE}

\label{sec:calibrate}
In order to make sense of the meaning of an RMSE score, we calibrate it to the effect of some easier to understand distortions.  
To start, we make a copy of $A$ (the default \s{G(W)} embedding -- we use this notation to signify a \GloVe embedding \s{G($\cdot$)} or the default Wikipedia corpus \s{W})) and apply an arbitrary rotation, translation, and scaling of it to obtain a new embedding $B$.  Invoking $\hat A, \hat B \leftarrow \AOsc(A,B)$, we expect that $\mathsf{RMSE}(\hat A, \hat B) = 0$;  we observe RMSE values on the order of $10^{-14}$, indeed almost $0$ withstanding numerical rounding.  



\myParagraph{Gaussian Noise}
Next we add Gaussian noise directly to the embedding.  That is we define an embedding $B$ so that each $b_i = a_i + g_i$ where $g_i \sim N_d(0, \sigma I)$, where $N_d(\mu,\Sigma)$ is a $d$-dimensional Gaussian distribution, and $\sigma$ is a standard deviation parameter.  
Then we measure $\mathsf{RMSE}(\hat A, \hat B)$ from 
$\hat A, \hat B \leftarrow \textsc{\AOR}(A,B)$.
Figure \ref{fig:RMSE-noise}(left) shows the effects for various $\sigma$ values, and also when only added to $10\%$ and $50\%$ of the points.  
We observe the noise is linear, and achieves an RMSE of $2$ to $5$ with $\sigma \in [0.1,0.3]$.  


\myParagraph{Noise before embedding.}
Next, we append noisy, \emph{unstructured} text into the Wikipedia dataset with $1$ billion tokens.  We specifically do this by generating random sequences of $m$ tokens, drawn uniformly from the $n = 10$K most frequent words; we use $m = \{0.01, 0.1, 0.5, 1, 2.5\}$ billion.  
We then extract embeddings for the same vocabulary of $n=100$K words as before, from both datasets, and use \AOR{} to linearly transform the noisy one to the one without noise.  
As observed in Figure \ref{fig:RMSE-noise}(middle), this only changes from about $0.7$ to $1.6$ RMSE.  
The embeddings seem rather resilient to this sort of noise, even when we add more tokens than the original data.  

We perform a similar experiment of adding structured text; we repeat a sequence made of $s = \{100, 1000, 10{,}000 \}$ tokens of medium frequency so the total added is again $m = \{10M, 100M, 500M, 1B,$ $2.5B\}$.  Again in Figure \ref{fig:RMSE-noise}(middle), perhaps surprisingly, this only increases the noise slightly, when compared to the unstructured setting.  
This can be explained since only a small percentage of the vocabulary is affected by this noise, and by comparing to the Gaussian noise, when only added to $10\%$ of the data, it has about a third of the RMSE as when added to all data.

%
%

\myParagraph{Incremental Data}
As a model sees more data, it is able to make better predictions and calibrate itself more acurately. This comes at a higher cost of computation and time. If after a certain point, adding data does not really affect the model, it may be a good trade off to use a smaller dataset to make an embedding almost equivalent to the one the larger dataset would produce.

We evaluate this relationship using the RMSE values when a \GloVe embedding from a smaller dataset $B$ is incrementally aligned to larger datasets $A$ using \AOR. We do this by starting off with a dataset of the first $1$ million tokens of Wikipedia (1M). We then add data sequentially to it, to create datasets of sizes of 100M, 1B, 2.5B or 4.57B tokens.  
For each dataset we create \GloVe embeddings. Then we align each dataset using $\AOR(A,B)$ where $A$ (the target) is always the larger of the two data set, and $B$ (the source) is rotated and is the smaller of the two.  

Figure \ref{fig:RMSE-noise}(right)  shows the result using a vocabulary of $n=100$K and $n' = 10$K.  The small $n'$ is also used since for smaller datasets, many of the top $100$K words are not seen.  We observe that even this change in data set size, decreasing from $4.57$B tokens to $2.5$B still results in substantial RMSE.  However aligning with fewer but better represented words starts to show better results, supporting use of weighted variants.   


\begin{table}[b]
	\caption{\label{tbl:RMSE-Datasets+scale}RMSE after alignment for embeddings.  
		Top: Created from different datasets.  
		Bottom: Created by different embeddings; uses \AOs {} mapping rows onto columns and changing scale.  }
	\centering
	\small
	\begin{tabular}{r||cccc}
		\hline
		$\rightarrow$ & \s{G(W)}  & \s{\hspace{-2mm}G(WG)}  & \s{\hspace{-2mm}G(CC42)} & \s{\hspace{-2mm}G(CC840)\hspace{-2mm}}\\ \hline
		\s{G(W)}  & -   & 4.56 & 5.167 & 6.148\\
		
		\s{G(WG)} & 4.56  & - & 5.986  & 6.876\\ 
		\s{\hspace{-1mm}G(CC42)} & 5.167  & 5.986 & -  & 5.507\\ 
		\s{\hspace{-1mm}G(CC840)} & 6.148 & 6.876 & 5.507 & - \\ \hline	
	\end{tabular}		
	\hspace{.25in}
	\centering
	\begin{tabular}{r||ccc}
		\hline
		$\rightarrow$ & \RAW  & \GloVe & \hspace{-1mm}\WordToVec \\ \hline
		\RAW  & -   & 4.12 & 14.73 \\
		
		\GloVe & 0.045 & - & 12.93\\
		\WordToVec &0.043  & 3.68 & -  \\ \hline
		\hspace{-2mm} Scale to \GloVe  & 25 & 1 & 0.25 \\
		\hspace{-2mm} Scale from \GloVe  & 0.011  & 1 & 3 \\ \hline
	\end{tabular}
	\normalsize
\end{table}

\begin{table*}[t]
	\caption{\label{tbl:Sim-Analogy}Spearman coefficient scores for synonym and analogy tests between the aligned \WordToVec to \GloVe embeddings and between \GloVe embeddings of Wikipedia and CC42 dataset; r,s and t stand for the functions of optimal rotation, scaling and translation respectively and w() is the weighted version of that function. When computing these scores across two embeddings, the best values are printed in bold. }
	\centering
	
	\small
	\begin{tabular}{r||cccccccccc}
		\hline
		\multirow{2}{*}{Test Sets} & \multirow{2}{*}{\GloVe} & \multirow{2}{*}{\WordToVec } & \multicolumn{8}{c}{\WordToVec to \GloVe} 
		\\
		\cline{4-11}
		& & & untransformed & r & r + s & r + t & r + s + t & normalized & w(r) & w(r+s+t)
		\\
		\hline
		\RG & 0.614 & 0.696 & 0.041 & 0.584 & 0.584 & 0.570 & 0.594 & 0.553 & 0.592 & \textbf{0.597}\\
		\WSim & 0.623 & 0.659 & 0.064 & 0.624 & 0.624 & 0.611 & 0.652 & 0.604 & 0.657 & \textbf{0.664}\\
		\MC &0.669 & 0.818 & 0.013 & 0.868 & 0.868 & 0.843 &0.873 & 0.743 & 0.878 & \textbf{0.882}\\
		\Simlex & 0.296 & 0.342 & 0.012 & 0.278 &0.278 & 0.269 &0.314 & 0.261 & 0.314 & \textbf{0.316}\\ \hline
		\SYN & 0.587 & 0.582 & 0.000 & 0.501 &0.525 & 0.517 & 0.528 &0.493 &0.535 & \textbf{0.539}\\
		\SEM & 0.691 & 0.722 & 0.0001 & 0.624 & 0.656 & 0.633 & 0.697 &0.604 & 0.702 & \textbf{0.712}\\
		\hline
	\end{tabular}

	\vspace{0.5mm}
	
	\begin{tabular}{r||cccccccccc}
		\hline
		\multirow{2}{*}{Test Sets} & \multirow{2}{*}{\s{G(W)}} & \multirow{2}{*}{\s{G(CC42)}} & \multicolumn{8}{c}{\s{G(W)} to \s{G(CC42)}}  
		\\
		\cline{4-11}
		& & & untransformed & r &  r + s & r + t & r + s + t  & normalized & w(r) & w(r+s+t)
		\\
		\hline
		\RG & 0.614 & 0.817 & 0.363 & 0.818 & 0.818 & 0.811 &0.821 & 0.815 &0.818 & \textbf{0.825}\\
		\WSim & 0.623 & 0.63 & 0.017 & 0.618 & 0.618 & 0.615 &0.618& 0.601 &0.616 & \textbf{0.637}\\
		\MC & 0.669 & 0.786 & 0.259 & 0.766 & 0.766 & 0.732 &0.768 & 0.705 &0.771 & \textbf{0.774}\\
		\Simlex & 0.296 & 0.372 & 0.035 & 0.343 & 0.343 & 0.339 & \textbf{0.346} & 0.296 & \textbf{0.346} & \textbf{0.346}\\ \hline
		\SYN & 0.587  & 0.625 &0.00 & 0.566  &  \textbf{0.576} & 0.572 & \textbf{0.576} & 0.502 &\textbf{0.576} & \textbf{0.576}\\
		\SEM & 0.691  &  0.741 &0.00& 0.676  &  0.684 & 0.676 & 0.688 & 0.565 & 0.690 & \textbf{0.695}\\
		\hline
	\end{tabular}
	\normalsize
\end{table*}

\subsection{Changing Datasets And Embeddings}

\label{sec:RMSE-data}

Now with a sense of how to calibrate the meaning of RMSE, we can investigate the effect of changing the dataset entirely or changing the embedding mechanism.

\myParagraph{Dependence of Datasets}
Table \ref{tbl:RMSE-Datasets+scale}(top) shows the RMSE when the $4$ \GloVe embeddings are aligned with \AOR, either as a target or source.  
The alignment of \s{G(W)} and \s{G(WG)} has less error than either to \s{G(CC42)} and \s{G(CC840)}, likely because they have substantial overlap in the source data (both draw from Wikipedia).  In all cases, the error is roughly on the scale of adding Gaussian noise with $\sigma \in [0.25, 0.35]$ to the embeddings, or reducing the dataset to $10$M to $100$M tokens. 
This is much more alignment error than in other experiments, indicating that the change in the source data set (and likely its size) has a much larger effect than the embedding mechanism.

\myParagraph{Dependence on Embedding Mechanism}
We now fix the data set (the default $4.57$B Wikipedia dataset \s{W}), and observe the effect of changing the embedding mechanism: using \GloVe, \WordToVec, and \RAW.  
We now use \AOs{} instead of \AOR, since the different mechanisms tend to align vectors at drastically different scales.

Table \ref{tbl:RMSE-Datasets+scale}(bottom) shows the RMSE error of the alignments; the columns show the target ($A$) and the rows show the source dataset ($B$).  
This difference in target and source is significant because the scale inherent in these alignments change, and with it, so does the RMSE. Also as shown, the scale parameter $s^*$ from \GloVe to \WordToVec in \AOs {} is approximately $3$ (and non-symmetrically about $0.25$ in the other direction from \WordToVec to \GloVe). This means for the same alignment, we expect the RMSE to be between $3$ to $4$ ($\approx 1/0.25$) times larger as well.  

However, with each column, with the same target scale, we can compare alignment RMSE.  We observe the differences are not too large, all roughly equivalent to Gaussian noise with $\sigma = 0.25$ or using only $1$B to $2.5$B tokens in the dataset.  
Interestingly, this is less error that changing the source dataset; consider the \GloVe column for a fair comparison.  This corroborates that the embeddings find some common structure, capturing the same linear structures, analogies, and similarities.  And changing the datasets is a more significant effect.

%
%
%

\subsection{Similarity and Analogies after Alignment}
\label{sec:similarity}
The \GloVe and \WordToVec embeddings both perform well under different benchmark similarity and analogy tests. These results will be unaffected by rotations or scaling.  
Here we evaluate how these tests transfer under alignment.  Using the default Wikipedia dataset, we use several variants of \AO{} to align \GloVe and \WordToVec embeddings.  Then given a synonym pair $(i,j)$ we check whether $b_j \in B$ (after alignment) is in the neighborhood of $a_i$.

More specifically, we use $4$ common similarity test sets, which we measure with cosine similarity~\cite{Lev2}: 
Rubenstein-Goodenough (\RG, 65 word pairs) \cite{rg}, 
Miller-Charles (\MC, 30 word pairs) \cite{mc}, 
WordSimilarity-353 (\WSim, 353 word pairs) \cite{wsim} and 
SimLex-999 (\Simlex, 999 word pairs) \cite{simlex}.
We use the Spearman correlation coefficient (in $[-1,1]$, larger is better) to aggregate scores on these tests; it compares the ranking of cosine similarity of $a_i$ to the paired aligned word $b_j$, to the rankings from a human generated similarity score.  

Table \ref{tbl:Sim-Analogy} shows the scores on just the \GloVe and \WordToVec embeddings, and then across these aligned datasets.  
To understand how the variants of \AO{} compare, we compute the scores after each of the various optimal transformation types are applied: rotation, then scaling, then translation, and finally we consider if we normalize all vectors before alignment to maximize cosine similarities. 
Before transformation (``untransformed'') the across-dataset comparison is very poor, close to $0$; that is, extrinsically there is very little information carried over.  However, alignment with just \AOR{} achieves scores nearly as good as, and sometimes better than on the original datasets.   \WordToVec scores higher than \GloVe, and the across-dataset scores are typically between these two scores.  Adding scaling with \AOs{} has no affect on the scores on the similarity test because they are measured with cosine similarity.  However also applying the optimal translation does increase the scores even though it optimizes Euclidean distance and not cosine distance.  Perhaps surprisingly, applying rotation along with translation \emph{and} scaling improves more than just applying rotation and translation.  This method applies scaling after the dataset is centered, so this then alters the inner products, and in a useful way.  

We perform the same experiments on $2$ Google analogy datasets~\cite{Mik1}:
\SEM has $8869$ analogies and
\SYN has $10675$ analogies. 
These are of the form ``\s{A:B::C:D}'' (e.g., ``man:woman::king:queen''), and we evaluate across data sets by measuring if vector $\ve{D}$ is among the nearest neighbors in data set $A$ of vector $\ve{C} + (\ve{B} - \ve{A})$ in data set $B$.  
The results are similar to the synonym tests, where \AOR{} alignment across-datasets performs similar to within either embedding, and scaling and rotation provided small further improvement.  In this case, performing rotation and scaling improves upon just rotation.  This is because the analogies are accessing something more complicated about the embedding, and so adjusting the scale more aligns the Euclidean distance and hence the vector structure needed to succeed in analogies.  

The right part of the table shows the effect of various weightings.  Normalization makes the similarity and analogy scores worse, but weighting by the norms consistently increases the scores.  Moreover, also scaling and rotating (e.g., as w(r+s+t)) improves the scores further.  

We also align \s{G(W)} to \s{G(CC42)}, to observe the effect of only changing the dataset.   The \s{G(CC42)} dataset performs better itself; it uses more data.  The small similarity tests (\s{RG},\s{MC}) show some extrinsic information is captured without any alignment, but otherwise across-embedding scores have a similar pattern to across-dataset scores.

Next in Table \ref{tbl : normalized} we further investigate the effect of various weighting (or normalizing) before alignment.  In these test we show the effect on \AOR{} with three types of weighting.  As before we simply apply \AOR{} on all 100K words.  But we also find the optimal $R$ on only the most frequent 10K words using \AOR{}, and then again using \AON{} on just these 10K words.  The rotation and evaluation is still on all 100K words needed for the tests.  Surprisingly \AON(10K) performs better than \AOR(10K), and comparably to \AOR(100K).  This indicates that similarity optimization is useful when the words all have sufficient data to embed them properly.  



\begin{table}
	
	\caption{Synonym and Analogy scores from rotations (applied to all words) learned on 100K, top 10K words, and top 10K words normalized.}
	\label{tbl : normalized}
	\centering
	\small
	\begin{tabular}{c||ccc}
		
		\hline
		&  \textsc{AO+R} & \textsc{AO+R} & \AON \\ 
		&  100K & 10K & 10K \\ \hline
		RG & 0.584 & 0.576 & 0.588 \\
		WSIM & 0.624 & 0.612 & 0.643\\
		MC & 0.868 & 0.817 & 0.851\\
		SIMLEX & 0.278 & 0.292 & 0.308\\ \hline
		SYN & 0.501 & 0.505 & 0.511\\
		SEM & 0.624 & 0.616 & 0.616\\
		\hline
	\end{tabular}
	
	\normalsize
\end{table}



\begin{table}[t]
	\caption{\label{tbl:baseline}Modified similarity tests (on only top 10K words) after alignment by Affine Transformation (AffTrans), LRA, and \AOR {} of Wikipedia and CC42 \GloVe embeddings. }
	\centering
	\small
	\begin{tabular}{r||ccc|c}
		\hline
		& LRA & AffTrans & \textsc{AO+R} & \textsc{AO+R} \\
		Test Sets & 10K & 10K & 10K & 100K \\
		\hline
		\RG & 0.701 & 0.301 & 0.728 & 0.818\\
		\WSim & 0.616 & 0.269 & 0.612 &0.618\\
		\MC & 0.719 & 0.412 & 0.722 & 0.766\\
		\Simlex & 0.327 & 0.126 & 0.340 & 0.343\\
		\hline 
	\end{tabular}
	
	\normalsize
\end{table}

\subsection{Comparison to Baselines }
Next, we perform similarity tests to compare against alignment implementations of methods by Sahin \etal (2017)~\cite{Sahin} (LRA) and Bollegela \etal (2017)~\cite{Art1} (Affine Transformations).  We reimplemented their algorithms, but did not spend significant time to optimize the parameters; recall our method requires no hyperparameters.  
We only used the top $n' = 10$K words for these transformations because these other methods were much more time and memory intensive.  
We only computed similarities among pairs in the top $10$K words for fairness (about two-thirds of the word pairs evaluated, so the scores do not match other tables), and did not perform analogy tests since fewer than one-third of analogies fully showed up in the top $10$K.  
Table \ref{tbl:baseline} shows results for aligning the \s{G(W)} and \s{G(CC42)} embeddings with these approaches.  
Our \AO-based approach does significantly better than Bollegela \etal's (2017)~\cite{Art1} Affine Transformations and generally better than Sahin \etal's (2017) \cite{Sahin} LRA.  Our advantage over LRA increases when aligning all $n=100$K words; by comparison LRA ran out of memory since it requires an $n \times n$ dense matrix decomposition.

\begin{table*}[t]
	\label{tbl:word freq}
	\caption{RMSE variation with word frequency in (a) \GloVe Wiki to \GloVe Common Crawl and (b) \WordToVec to \GloVe evaluated for Wiki dataset.}
	\centering
	\begin{tabular}{rr||ccc}
		\hline
		Word & Frequency in Wiki  & Norm in (\GloVe) Wiki & Wiki to CC (42B)  & \WordToVec to \GloVe
		\\ \hline 
		talk & 187513532 & 9.681& 8.608 & 12.462 \\
		november &2340726 & 7.847 & 5.26 &  8.614 \\
		man & 1035008 & 8.648 & 4.25 & 5.161\\
		statistical & 83531 & 5.891 & 4.63& 5.097\\
		bubbles & 11200 & 5.455 & 4.66 & 3.768 \\
		skateboard & 3804 & 5.670 & 5.714& 3.891 \\
		emoji & 1761 & 6.090 & 6.781 & 2.402\\
		haymaker & 705 & 4.108 & 5.951  & 1.573\\
		\hline
	\end{tabular}
\end{table*}

\begin{table*}
	\caption{\label{tbl:boosting weighted}Similarity and Analogy tests before and after alignment and combining embeddings derived from different techniques and datasets by \AOWts.  Best scores in {\bf bold}.}
	\centering
	\small
	\begin{tabular}{rccccccc}
		\hline
		TestSets & \s{G(W)}  & \hspace{-1mm}\s{W(W)} & \hspace{-2mm}[\s{G(W)}$\odot$\s{W(W)}]\hspace{-1mm} & \s{W(GN)} & \hspace{-2mm}[\s{G(W)}$\odot$\s{W(GN)}]\hspace{-1mm} & \hspace{-1mm}\s{G(CC840)} & \hspace{-2mm}[\s{G(CC840)}$\odot$\s{W(GN)}]\hspace{-2mm} \\
		\hline
		\RG & 0.614 & 0.696 & 0.716 & 0.760 &   {\bf 0.836} & 0.768 & 0.810 \\
		\WSim & 0.623 & 0.659 & 0.695 & 0.678 & 0.708 & 0.722 &  {\bf 0.740} \\
		\MC & 0.669 & 0.818 & {\bf 0.869} & 0.800 & 0.811 & 0.798 & 0.847 \\
		\Simlex & 0.296 & 0.342 & 0.368 & 0.367 & 0.394 & 0.408 &  {\bf 0.446} \\
		\hline 
		\SYN & 0.587& 0.582  & 0.592 & 0.595 & 0.607 &  {\bf 0.618} & 0.609 \\
		\SEM & 0.691 & 0.722 &  {\bf 0.759} &  0.713 & 0.733 & 0.729 & 0.733 \\
		\hline
	\end{tabular}
	\normalsize
	\vspace{-3mm}		
\end{table*}

\subsection{Dependence of RMSE variation with Word Frequency }
\label{sec:word-freq}

Table \ref{tbl:word freq} shows some sampled words of various frequencies in the Wikipedia data set.  A word that is more frequently seen in a corpus is generally seen with a larger proportion of other words and contexts, and thus as observed in the table, has a vector representation that has larger norm than a word which has low frequency.  This results in the contribution of high frequency words in the rotation matrix $H$, computed for minimizing the RMSE, to also be larger.  
This larger frequency, and larger norm, also manifests itself in the error after alignment, as shown in the last two columns of Table \ref{tbl:word freq}, both between data sets and between embedding mechanisms.  The relation in the amount of RMSE between words appears even more correlated when between embedding mechanisms (in this case \WordToVec and \GloVe).  The low-frequency words likely exhibit some baseline noise in the case with different data sets (Wiki and CC(42B)), which obscures this relationship for low-frequency words.

\subsection{Discussion on the Right Variant}
\label{sec:which-AO}


Most of the gain using \AO{} is achieved by just finding the optimal rotation $R$ with \AOR.  
However, consistent improvement can be found by weighting the large points more using \AOW{} and by applying translation or scaling, and slightly more by applying both.  

When different datasets are aligned using the same mechanism (e.g., both with \GloVe or both with \WordToVec), then it is debatable whether scaling and translation is necessary, since scaling does not affect cosine similarity, and translation changes intrinsic inner product properties.  However, using a weighting to put more weight on longer (and implied more robustly embedded) words does not alter any intrinsic properties, and only seems to create better alignments.  

When datasets are embedded with different mechanisms (e.g., one with \WordToVec and one with \GloVe) then they are not scaled properly with respect to each other.  In this case, it is important to find the optimal scaling to put them in a consistent interpretable scale, and to ensure analogy relations are optimized.  So we strongly recommend using scaling in this setting.  

%
%
%


\section{EMBEDDING ALIGNMENT : APPLICATIONS}

\label{app : Applications}

We highlight a few applications which may be served by this alignment, and comparison mechanisms that we design and demonstrate their effectiveness.

\subsection{Boosting via Ensembles}

\label{sec:boosting}
A direct application of combining different embeddings can be to increase its robustness.  We show that ensembles of pre-computed word embeddings found via different mechanisms and on different datasets can boost the performance on the similarity and analogy tests beyond that of any single mechanism or dataset.  The boosting strategy we use here is just simple averaging of the corresponding words after the embeddings have been aligned.  

Table \ref{tbl:boosting weighted} shows the performance of these combined embedding in three experiments.  
The first set shows the default Wikipedia data set under \GloVe (\s{G(W)}), under \WordToVec (\s{W(W)}), and combined ([\s{G(W)}$\odot$\s{W(W)}]).  
The second set shows \WordToVec embedding of GoogleNews (\s{W(GN)}), and combined ([\s{G(W)}$\odot$\s{W(GN)}]) with \s{G(W)}.  
The third set shows \GloVe embedding of CommonCrawl (840B) (\s{G(CC840)}) and then combined with \s{W(GN)} as [\s{G(CC840)}$\odot$\s{W(GN)}].  
Combining two embeddings using \AOWts{} consistently boosts the performance on similarity and analogy tests.  Very similar boosting results occur independent of the precise alignment mechanism (e.g., using \AOsc).  
The best score on each experiment is in bold, and in 5 out of 6 cases, it is from a combined embedding.  Moreover, except for this one case, the combined embedding always performs better on all tests that both of the individual embeddings, and in this one case, \s{G(CC804)$\odot$W(GN)} still outperforms \s{W(GN)} on \SEM analogies.  
For instance, remarkably, \s{G(W)$\odot$W(W)} which only uses the default $4.57B$ token Wikipedia dataset, performs better or nearly as well as \s{W(GN)} which uses $100B$ tokens.  Moreover, in some cases the improvement is significant; on the large similarities test \Simlex, the  [\s{G(CC840)}$\odot$\s{W(GN)}] score is $0.443$ or $0.446$ with weights, whereas the best score without boosting is only $0.408$ using \s{G(CC840)}.

\begin{table*}[h]
	\caption{\label{tbl:translation-spanish}
		The 5 closest neighbors of a word before and after alignment by \AOR (between English - Spanish). Target word (translation) in \textbf{bold}.}
	\centering 
	\small
	\begin{tabular}{r||cc}
		\hline
		Word  & Neighbors before alignment & Neighbors after alignment  \\
		\hline
		woman & her, young, man, girl, mother & her, girl, \textbf{mujer}, mother, man\\
		week & month, day, year, monday, time &  days, \textbf{semana}, year, day, month\\
		casa & apartamento, casas, palacio, residencia, habitaci & casas, \textbf{home}, homes, habitaci, apartamento\\
		caballo & caballos, caballer, jinete, jinetes, equitaci & \textbf{horse}, horses, caballos, jinete\\
		sol & sombra,luna,solar,amanecer,ciello & \textbf{sun}, moon, luna, solar, sombra \\
		\hline
	\end{tabular}
	\normalsize
\end{table*}

\begin{table*}[h]
	\caption{The 5 closest neighbors of a word before and after alignment by \AO (between English - French). Target word (translation) in \textbf{bold}.}
	\centering 
	\small
	\begin{tabular}{r||cc}
		\hline
		Word  & Neighbors before alignment & Neighbors after alignment  \\
		\hline
		woman & her, young, man, girl, mother & her, young, man, \textbf{femme}, la\\
		week & month, day, year, monday, time & month, day, year, \textbf{semaine}, start\\
		heureux & amoureux, plaisir, rire, gens, vivre & \textbf{happy}, plaisir, loving, amoureux, rire\\
		cheval & chein, petit, bateau, pied, jeu & \textbf{horse}, dog, chien, red, petit\\
		daughter & father, mother, son, her, husband & mother, \textbf{fille}, husband,  mere, her \\
		\hline
	\end{tabular}
	\normalsize
	\label{tbl : translation}
\end{table*}

\subsection{Aligning Embeddings Across Languages And Embeddings }
\label{app : langauges}

Word embeddings have been used to place word vectors from multiple languages in the same space~\cite{Herm,Mik3}. These either do not perform that well in monolingual semantic tasks as noted in Luong, Pham and Manning ~\cite{Luong} or use learned affine transformations~\cite{Mik3}, which distort distances and do not have closed form solutions.  Smith \etal \cite{SmithTHH17} use the equivalent of \AOR{} to translate between word embeddings from different languages that have been extracted using the same method. We extend that here to verify that no matter the embedding mechanism, we can translate using a variant of \AO{}. We use the ability to choose the right variant of Absolute Orientation as per Section \ref{sec:which-AO} to orient different embeddings onto each other coherently.
We use the default English \GloVe embedding from Wikipedia and the FastText \url{https://github.com/facebookresearch/fastText} embedding for Spanish.  FastText is yet another unsupervised learning paradigm for obtaining vector representations for words which uses a lot of concepts from \WordToVec, skipgram models and bag of words. As presented, these two have been derived using different methods and are thus oriented differently in 300 dimensional space. We extract the embeddings for the most frequent $5{,}000$ words from the default English Wikipedia dataset (that have translations in Spanish) and their translations in Spanish and align them using \AOWts. We test before and after alignment, for each of these $10{,}000$ words, if their translation is among their nearest $1$, $5$, and $10$ neighbors.
Before alignment, the fraction of words with its translation among its closest $1$, $5$ and $10$ nearest neighbors is $0.00$, $0.160$, and $0.160$ respectively, while after alignment it is $0.372$, $0.623$ and $0.726$, respectively. 


We perform a cross-validation experiment to see how this alignment applies to new words not explicitly aligned.  On learning the rotation matrix above, we apply it to a set of $1000$ new 'test' Spanish words (the translations of the next $1000$ most frequent English words) and bring it into the same space as that of English words as before.  We test these $2000$ new words in the embedded and aligned space of $12{,}000$ words (now $6{,}000$ from each language).  
Before alignment, the fraction of times their translations are among the closest $1$, $5$ and $10$ neighbors are $0.00$, $0.00$, and $0.00$, respectively. After alignment it is $0.311$, $0.689$, and $0.747$,  respectively (comparable to results and setup in Mikolov \etal~\cite{Mik3}, using jointly learned affine transformations).

We perform a similar experiment between English and French, and see similar results. We first obtain $300$ dimensional embeddings for English Wikipedia dump using \GloVe, and for French words from the FastText embeddings.   
Then, we extract the embeddings for the most frequent $10{,}000$ words from the default Wikipedia dataset (that have translations in French) and their translations in French and align them using \AOWts.  We test before and after alignment, for each of these $10{,}000$ words, if their translation is among their nearest $1$, $5$, and $10$ neighbors.
Before alignment, the fraction of words with its translation among its closest $1$, $5$ and $10$ nearest neighbors is $0.00$, $0.054$, and $0.054$ respectively, while after alignment it is $0.478$, $0.755$ and $0.810$, respectively. 
Table \ref{tbl : translation} lists some examples before and after translation.

We again perform a cross-validation experiment to see how this alignment applies to new words not explicitly aligned.  On learning the rotation matrix above, we apply it to a set of $1000$ new 'test' French words (the translations of the next $1000$ most frequent English words in the default dataset) and bring it into the same space as that of English words as before.  We test in this space of $22{,}000$ words now, if their translations are among the closest $1$, $5$ and $10$ nearest neighbors of the $2000$ new words ($1000$ French and their translations in English).  Before alignment, the fraction of times their translations are among the closest $1$, $5$ and $10$ neighbors are $0.00$, $0.00$, and $0.00$, respectively.   After alignment it is $0.307$, $0.513$, and $0.698$,  respectively.

\subsection{Aligning Multiple Languages onto Same Space}

\update{As demonstrated in Section \ref{app : langauges}, pairwise alignment of words from different languages needs relatively few points to find the alignment to achieve good accuracy in translation between the two languages for a much larger set of words. This allows us to have a low cost operation to map words of one language to their corresponding translated words in the another language. This additionally leads us to a follow-up application. For many language pairs (say languages $L_1$ and $L_2$), we might not have a known dictionary of corresponding word-pairs. In such cases, finding an alignment for enabling translation can be impeded. However, for each of these languages $L_1$ and $L_2$, if corresponding words to a third language $L_3$ is known, aligning both $L_1$ and $L_2$ onto $L_3$ also brings $L_1$ and $L_2$ into the same space. Thus, translation of words from $L_1$ and $L_2$ is enabled without having a set of corresponding seed words in them by which to define the alignment. 
Aligning multiple languages onto the same space can thus, aid in multi-way translation. Further, for low resource languages or pairs of languages for whom, only a very small set of translations, i.e., few corresponding points are known, aligning each of these languages to a more common language with which a larger correspondence is known, can help translation.}


\update{To demonstrate this, we pick languages $L_1$ and $L_2$ to be Spanish and French respectively. We also pick the common language $L_3$ to be English, to whose word embedding space we align $L_1$ and $L_2$ to. In Table \ref{tbl : multi}, in the first column, we have Spanish to French translations before alignment. As expected, the top $1$, $5$ and $10$ neighboring word accuracies (as evaluated in section \ref{app : langauges}) are poor (in fact $0$ accuracy).} In the second column, we have accuracies after aligning them onto each other using a pool of 2000 words for which we know translations, i.e., their one-to-one correspondences. Next, in the third column, we align both Spanish and French onto English, using the same set of 2000 words and then compare the accuracies for translations from Spanish to French. We find that the top $1$, $5$ and $10$ accuracies are comparable between columns 2 and 3. 
\update{Thus Spanish-French translation was enabled by knowing Spanish-English and French-English associations. This multi-way translation enabled by a third language's association leads us to many possibilities of aligning low-resource languages to each other easily.}   

\begin{table}[t]
    \caption{Translating Spanish to French by aligning directly as compared to aligning both to English}
    \centering
    \begin{tabular}{c||ccc}
        \hline
         Top n accuracy & Unaligned & Pairwise Aligned & Indirectly Aligned  \\
         \hline 
         n = 1 & 0.0 & 0.312 & 0.277\\
         n = 5 & 0.0 & 0.635 & 0.601\\
         n = 10 & 0.0 & 0.723 & 0.707\\
        \hline
    \end{tabular}
\label{tbl : multi}
\end{table}


\section{Discussion}
\label{sec:discussion}
We have provided simple, closed-form method to align word embeddings.  
Code can be found on github (\url{https://github.com/sunipa/Abs-Orientation}).  	
It allows for transformations for any subset of translation, rotation, and scaling.  These operations all preserve the intrinsic Euclidean structure which has been shown to give rise to linear structures which allows for learning tasks like analogies, synonyms, and classification.  All of these operations also preserve the Euclidean distances, so it does not affect the tasks which are measured using this distance; note the scaling also scales this distance, but does not change its order.  Our experiments indicate that the rotation is essential for a good alignment, and the scaling is needed to compare embeddings generated by different mechanisms (e.g., \GloVe and \WordToVec) and while helpful, not necessarily when the data set is changed.  Also the translation provides minor but consistent improvement.  

We also show how to explicitly optimize cosine similarity by first normalizing all words -- however, this does not perform as well as instead optimizing Euclidean distance.  Rather we propose to weight words in the alignment by their norms, and this further improves the alignment because it emphasize the words which have more stable embeddings.  

This alignment enables new ways that word embeddings can be compared.  This has the potential to shed light on the differences and similarity between them.  For instance, as observed in other ways, common linear substructures are present in both \GloVe and \WordToVec, and these structures can be aligned and recovered, further indicating that it is a well-supported feature inherent to the underlying language (and dataset).  We also show that changing the embedding mechanism has less of an effect than changing the data set, as long as that data set is meaningful.  Unstructured noise added to the input data set appears not to have much effect, but changing from the $4.57$B token Wikipedia data set to the $840$B token Common Crawl data set has a large effect.  

Additionally, we show that by aligning various embeddings, their characteristics as measured by standard analogy and synonym tests can be transferred from one embedding to another.  We also demonstrate that cross-language alignment can aid in word translation even when coming from completely different embedding mechanisms, even in a cross-validation setting.  This cross embedding-mechanism alignment opens the door for many other types of alignment word embeddings with embeddings generated from graphs, images, or any other data set which has some useful word labels.  

Finally, we showed that we can ``boost'' embeddings without revisiting the (sometimes quite enormous) raw data.  This is surprisingly effective in improving scores on similarity and analogy test, results in the best known scores from embeddings on these tests.  For instance, on the \Simlex analogy test we improve upon the best known scores by almost $10\%$ in the Spearman correlation coefficient.  
There are many other potential applications of these techniques for aligning high-dimensional data embeddings. We propose some scenarios where they may be used in the following section.

\subsection{Other Applications}
  Here, we enumerate a few applications-- we do not experiment on many of these due to the extreme computational cost of performing an analysis of the effect (i.e., the baseline approaches of not using our techniques can be prohibitively expensive, or too qualitative to effectively evaluate).

\begin{enumerate}
	\item Common Crawl is one of the largest textual data sources available.  Moreover, it consistently gets updated to include the ever increasing data on the internet. Each of these datasets has over 800B tokens, and extracting embeddings from these can be computationally expensive. However, extracting embeddings from the additional data not included in the previous update of Common Crawl should be significantly less expensive.  Aligning an embedding from just the new data, and performing a weighted average with the older larger one may work as well or better than the embedding made from scratch.  
	
	\item 
	A similar weighted average alignment can help with specialized data. Consider data from scientific journals only, or of domain-specific biomedical terms.  Embeddings from just these data sets would be very specialized and each words would have a specific word sense based on the domain.  Aligning these to a gigantic corpus can enrich the specialized domain related regions on the larger embedding. 
	
	\item 	Tags and phrases in English can be single words or a string of words. Orienting an embedding of tags/phrases along say Common Crawl using an intersection of the single words in the two datasets can help place multi-worded tags or phrases around words related to them. This can help derive meaning from random or unknown phrases. Images also often are annotated with a set of tag words.  So orienting a set of tags can help orient images meaningfully among words.
	
	\item 
	These methods can also be applied to even more heterogeneous embeddings than discussed above.  
	We can orient heterogeneous embeddings derived from a variety of methods e.g. for graphs including node2vec~\cite{node2vec} or DeepWalk~\cite{DeepWalk}, and others \cite{CZC17,GF17,DCS17}, images~\cite{SIFT,SURF}, and for kernel methods~\cite{RR07,RR08}.  
	For instance, RDF data can contain shorthand query phrases like 'president children spouse' which answers the question 'who are the spouses and children of presidents?'  By orienting each word along word embeddings from Common Crawl, this may help answer similar questions even more abstractly.  
	Heterogeneous networks have a mixture of node types. If there is an intersection of some nodes (and node types) between any two embeddings (heterogeneous or homogeneous), we can orient them meaningfully.
	
\item Customer data collected at a company over different years and subsidiaries can be embedded using different features (such as income bracket, credit score, location etc{.} depending on the company). Using common customers over the year, diverse sources and new users can be added meaningfully to the embedding and inferred about, without embedding all of the data points from scratch. 
Moreover, embedding the same users from different years and aligning them can also help deduce the change in their features over time.  
\end{enumerate}

\section{Acknowledgements}
Thanks to NSF CCF-1350888, ACI-1443046, CNS-1514520, CNS-1564287, IIS-1816149, and NVidia Corporation.

\bibliographystyle{plain}
\bibliography{icpbib}
\vspace{-5mm}

\end{document}